\def\year{2023}\relax
\documentclass[letterpaper]{article} 
\usepackage{aaai22}  
\usepackage{times}  
\usepackage{helvet}  
\usepackage{courier}  
\usepackage[hyphens]{url}  
\usepackage{graphicx} 
\usepackage{natbib}  
\usepackage{caption} 
\DeclareCaptionStyle{ruled}{labelfont=normalfont,labelsep=colon,strut=off} 
\usepackage[linesnumbered,ruled]{algorithm2e}
\usepackage{varioref}
\usepackage{graphicx}
\usepackage{subfig}
\usepackage{amsmath}
\usepackage{amsthm}
\usepackage{amssymb}
\usepackage{amsfonts}
\usepackage{thmtools} 
\usepackage{thm-restate}
\usepackage{stmaryrd}
\usepackage{paralist}
\usepackage{xcolor}
\usepackage{soul}
\usepackage[aboveskip = 6pt]{caption}
\usepackage{comment}
\usepackage{tabularx}
\usepackage{multirow}
\usepackage{booktabs}
\usepackage[figuresright]{rotating}
\usepackage{appendix}
\usepackage[colorinlistoftodos,prependcaption,textsize=tiny]{todonotes}
\usepackage[shortlabels]{enumitem}
\usepackage{longtable}

\newtheorem{example}{Example}

\graphicspath{{./figures/}{../Figures/}{../../Figures/}{../Figures/experiments/}{../Figures/variable-order/}{../Figures/OBDDCs/}}
\theoremstyle{plain}

\theoremstyle{definition}
\newtheorem{definition}{Definition}
\theoremstyle{remark}

\newcommand{\Floor}[1]{\lfloor{#1}\rfloor}
\newcommand{\Ceil}[1]{\lceil{#1}\rceil}
\newcommand{\Angle}[1]{\langle #1 \rangle}

\newcommand{\expectation}[2]{\ensuremath{\mathbb{E}_{#1}[#2]}}
\newcommand{\Expectation}[2]{\ensuremath{\mathbb{E}_{#1}\left[#2\right]}}

\newcommand{\lequ}{\leftrightarrow}

\newcommand{\D}{\mathcal{D}}

\newcommand{\OBDD}[1]{OBDD$_{#1}$}

\newcommand{\CCDD}{CCDD}

\newcommand{\NNF}{\ensuremath{\mathsf{NNF}}}
\newcommand{\DNNF}{\ensuremath{\mathsf{DNNF}}}
\newcommand{\dDNNF}{\ensuremath{\mathsf{d\text{-}DNNF}}}
\newcommand{\DecDNNF}{Decision-\ensuremath{\mathsf{DNNF}}}
\newcommand{\SDD}[1]{\ensuremath{\text{SDD}_{#1}}}

\newcommand{\Ctwo}{\ensuremath{\mathsf{CCDD}}}

\newcommand{\dfour}{D4}

\newcommand{\PartialKC}{\ensuremath{\mathsf{PartialKC}}}
\DontPrintSemicolon
\SetKwBlock{KwFunc}{function}{end}

\newcommand{\MargProb}{\textsc{MargProb}}

\newcommand{\ConstructCore}{\textsc{ConstructCore}}

\newcommand{\MicroKC}{\textsc{MicroKC}}
\newcommand{\ProjectedKC}{\textsc{ProjectedKC}}

\newcommand{\DetectLitEqu}{\textsc{DetectLitEqu}}

\newcommand{\Decompose}{\textsc{Decompose}}
\newcommand{\PickGoodVar}{\textsc{PickGoodVar}}

\newcommand{\Kernelizable}{\textsc{ShouldKernelize}}

\newcommand{\Vars}{\ensuremath{\mathit{Vars}}}
\newcommand{\Cache}{\ensuremath{\mathit{Cache}}}
\newcommand{\EasyInstance}{\textsc{EasyInstance}}

\title{Fast Converging Anytime Model Counting\protect\thanks{The author list has been sorted alphabetically by last name; this should not be used to determine the extent of authors' contributions.} }
\author{
        Yong Lai,\textsuperscript{\rm 1}
        Kuldeep S. Meel,\textsuperscript{\rm 2}
        Roland H. C. Yap\textsuperscript{\rm 2}\\
}
\affiliations{
    \textsuperscript{\rm 1} Key Laboratory of Symbolic Computation and Knowledge Engineering of Ministry of Education, Jilin University, China \\
    \textsuperscript{\rm 2} School of Computing, National University of Singapore
}

\begin{document}

\maketitle

\begin{abstract}
Model counting is a fundamental problem which has been influential in many applications, from artificial intelligence to formal verification.
Due to the intrinsic hardness of model counting, 
approximate techniques have been developed to solve
real-world instances of model counting.
This paper designs a new anytime approach called \PartialKC{} for approximate model counting.
The idea is a form of partial knowledge compilation to provide an unbiased estimate of the model count which can converge to the exact count.
Our empirical analysis demonstrates that {\PartialKC} achieves significant scalability and accuracy over prior state-of-the-art approximate counters, including satss and STS.
Interestingly, the empirical results show that \PartialKC{} reaches convergence for many instances and therefore provides exact model counting performance comparable to state-of-the-art exact counters.
\end{abstract}

\section{Introduction}\label{sec:intro}

Given a propositional formula $\varphi$, the model counting problem (\#SAT) is to compute the number of satisfying assignments of $\varphi$. 
Model counting is a fundamental problem in computer science which has
a wide variety of applications in practice, ranging from 
probabilistic inference~\cite{Roth:96,Chavira:Darwiche:08}, probabilistic databases~\cite{VdB:Suciu:17}, probabilistic programming~\cite{Fierens:etal:15}, neural network verification~\cite{Baluta:etal:19}, network reliability~\cite{Duenas-Osorio:etal:17}, computational biology~\cite{Sashittal:El-Kebir:19}, and the like.
The applications benefit significantly from efficient propositional model counters.

In his seminal work, Valiant~\shortcite{Valiant:79} showed that model counting is \#P-complete, where \#P is the set of counting problems associated with NP decision problems. 
Theoretical investigations of \#P have led to the discovery of strong evidence for its hardness.
In particular, Toda~\shortcite{Toda:89} showed that PH $\subseteq P^{\#P}$; that is, each problem in the polynomial hierarchy could be solved by just one call to a \#P oracle.
Although there has been large improvements in the scalability of practical exact model counters, the issue of hardness is intrinsic.
As a result, researchers have studied approximate techniques to solve
real-world instances of model counting.


The current state-of-the-art approximate counting techniques can be categorized into three classes based on the guarantees over estimates~\cite{ApproxMC}. 
Let $\varphi$ be a formula with $Z$ models.
A counter in the first class is parameterized by $(\varepsilon, \delta)$, and computes a model count of $\varphi$ that lies in the interval $[(1 + \varepsilon)^{-1}Z, (1 + \varepsilon)Z]$ with confidence at least $1 - \delta$.
ApproxMC \cite{ApproxMC,ApproxMC2} is a well-known counter in the first class. 
A counter in the second class is parameterized by $\delta$, and computes a lower (or upper) bound of $Z$ with confidence at least $1 - \delta$
including tools such as
MBound \cite{Gomes:etal:06} and SampleCount \cite{Gomes:etal:07}.
Counters in the third class provide weaker guarantees, but offer relatively accurate approximations in practice and 
state-of-the-art counters include satss \cite{Gogate:Dechter:11} and STS \cite{STS}.
We remark that some counters in the third class can be converted into a counter in the second class. Despite significant efforts in the development of approximate counters over the past decade, scalability remains  a major challenge. 

In this paper, we focus on the third class of counters to achieve scalability. To this end, it is worth remarking that  a well-known problem for this class of approximate model counters is slow convergence. Indeed, in our experiments, we found satss and STS do not
converge in more than one hour of CPU time for many instances, whereas an exact model counter can solve those instances in several minutes of CPU time.
We seek to remedy this.
Firstly, we notice that each sample generated by a model counter in the third class represents a set of models of the original CNF formula.
Secondly, we make use of knowledge compilation techniques to accelerate the convergences.
Since knowledge compilation languages can often be seen as a compact representation of the model set of the original formula; we infer the convergence of the approximate counting by observing whether a full compilation is reached.
By storing the existing samples into a compiled form, we can also speed up the subsequent sampling by querying the current stored representation. 

We generalize a recently proposed language called \Ctwo{} \cite{Lai:etal:21}, which was shown to support efficient exact model counting, to represent the existing samples. 
The new representation, partial \Ctwo{}, adds two new types of nodes called known or unknown respectfully representing, (i) a sub-formula whose model count is known; and (ii) a sub-formula which is not explored yet.
We present an algorithm, \PartialKC{}, to generate a random partial CCDD that provides an unbiased estimate of the true model count.
\PartialKC{} has two desirable properties for an approximate counter: 
(i) it is an anytime algorithm;
(ii) it can eventually converge to the exact count.
Our empirical analysis demonstrates that {\PartialKC} achieves significant scalability as well as accuracy over prior state of the art approximate counters, including satss and STS. 

The rest of the paper is organized as follows:  We present notations, preliminaries, and related work in  Sections \ref{sec:prelims}--\ref{sec:related}. 
We introduce partial CCDD in Section~\ref{sec:PCCDD}. In Section~\ref{sec:PartialKC}, we present the model counter, \PartialKC{}. 
Section~\ref{sec:experiments} gives detailed experiments, and 
Section~\ref{sec:conclusion} concludes.

\section{Notations and Background}\label{sec:prelims}

In a formula or the representations discussed,
$x$ denotes a propositional variable, and
literal $l$ is a variable $x$ or its negation $\neg x$, where $var(l)$
denotes the variable.
$\mathit{PV} = \{x_0, x_1, \ldots, x_n, \ldots\}$ denotes a
set of propositional variables. 
A formula is constructed from constants $\mathit{true}$, $\mathit{false}$, and propositional variables using the negation ($\lnot$), conjunction ($\land$), disjunction ($\lor$), and equality ($\lequ$) operators.
A clause $C$ is a set of literals representing their disjunction.
A formula in conjunctive normal form (CNF) is a set of clauses representing their conjunction.  
Given a formula $\varphi$, a variable $x$, and a constant $b$, a substitution $\varphi[x \mapsto b]$ is a transformed formula by replacing $x$ with $b$ in $\varphi$.
An assignment $\omega$ over a variable set $X$ is a mapping from $X$ to $\{true, \mathit{false}\}$. 
The set of all assignments over $X$ is denoted by $2^X$. 
A model of $\varphi$ is an assignment over $\Vars(\varphi)$ that satisfies $\varphi$; that is, the substitution of $\varphi$ on the model equals $\mathit{true}$. Given a formula $\varphi$, we use $Z(\varphi)$ to denote the number of models, and the problem of model counting is to compute $Z(\varphi)$.

\paragraph*{Sampling-Based Approximate Model Counting}

Due to the hardness of exact model counting, sampling is a useful technique to estimate the count of a given formula.
Since it is often hard to directly sample from the distribution over the model set of the given formula, we can use importance sampling to estimate the model count \cite{Gogate:Dechter:11,Gogate:Dechter:12}.
The main idea of importance sampling is to generate samples from another easy-to-simulate distribution $Q$ called the proposal distribution.
Let $Q$ be a proposal distribution over $\Vars(\varphi)$ satisfying that $Q(\omega) > 0$ for each model $\omega$ of $\varphi$.
Assume that 0 divided by 0 equals 0, and therefore $Z(\varphi) = \Expectation{Q}{\frac{Z(\varphi|_{\omega})}{Q(\omega)}}$.
For a set of samples $\omega_1, \ldots, \omega_N$, $\widehat{Z_N} = \frac{1}{N}\sum_{i = 1}^{N}\frac{Z(\varphi|_{\omega_i})}{Q(\omega_i)}$ is an \emph{unbiased} estimator of $Z(\varphi)$; that is, $\expectation{Q}{\widehat{Z_N}} = Z(\varphi)$.
Similarly, a function $\widetilde{Z_N}$ is an \emph{asymptotically} unbiased estimator of $Z(\varphi)$ if $\lim_{N \rightarrow \infty}\expectation{Q}{\widetilde{Z_N}} = Z(\varphi)$.
It is obvious that each unbiased estimator is asymptotically unbiased.

\paragraph*{Knowledge Compilation}

In this work, we will concern ourselves with the subsets of Negation Normal Form (\NNF{}) wherein the internal nodes are labeled with  conjunction ($\wedge$) or disjunction ($\vee$) while the leaf nodes are labeled with  $\bot$ ($\mathit{false}$), $\top$ ($\mathit{true}$), 
or a literal.
For a node $v$, we use $sym(v)$ to denote the labeled symbol, and $Ch(v)$ to denote the set of its children. We also use $\vartheta(v)$ and $\Vars(v)$ denote the formula represented by the DAG rooted at $v$, and the variables that label the descendants of $v$, respectively.  
We define the well-known decomposed conjunction~\cite{Darwiche:Marquis:02}
as follows:
\begin{definition}
	A conjunction node $v$ is called a \emph{decomposed conjunction} if its children (also known as conjuncts of $v$) do not share variables. That is, for each pair of  children $w$ and $w'$ of $v$, we have  $\Vars(w) \cap \Vars(w') = \emptyset$.
\end{definition}

If each conjunction node is decomposed, we say the formula is in \emph{Decomposable} \NNF{} (\DNNF{})~\cite{Darwiche:01a}. 
\DNNF{} does not support tractable model counting, but the following subset does:

\begin{definition}
	A disjunction node $v$ is called \emph{deterministic} if each two disjuncts of $v$ are logically contradictory. That is, any two different children $w$ and $w'$ of $v$ satisfy that $\vartheta(w) \land \vartheta(w') \models false$. 
\end{definition}

If each disjunction node of a \DNNF{} formula is deterministic, we say the formula is in deterministic \DNNF{} (\dDNNF{}). \emph{Binary decision} is a practical property to impose determinism in the design of a compiler (see e.g., \dfour{} \cite{D4}). 
Essentially, each decision node with one variable $x$ and two children is equivalent to a disjunction node of the form $(\lnot x \land \varphi) \lor (x \land \psi)$, where $\varphi$, $\psi$ represent the formulas corresponding to the children.
If each disjunction node is a decision node, the formula is in \DecDNNF{}. Each \DecDNNF{} formula satisfies the read-once property: each decision variable appears at most once on a path from the root to a leaf.

Recently, a new type of conjunctive nodes called \emph{kernelized} was introduced \cite{Lai:etal:21}. Given two literals $l$ and $l'$, we use $l \lequ l'$ to denote literal equivalence of $l$ and $l'$.  
Given a set of literal equivalences $E$, let $E' = \{l \lequ l', \lnot l \lequ \lnot l'\mid l \lequ l' \in E\}$; and then we define semantic closure of $E$,  denoted by $\Ceil{E}$, as the equivalence closure of $E'$. Now for every literal $l$ under $\Ceil{E}$, let $[l]$ denote the equivalence class of $l$. Given $E$, a unique equivalent representation of $E$, denoted by $\Floor{E}$ and called {\em prime literal equivalences}, is defined as follows: \\ 
\centerline{
	$\displaystyle
	\Floor{E} =   \bigcup\limits_{x \in \mathit{PV}, \min_\prec[x] = x}^{} \{ x \lequ l \mid l \in [x], l \neq x \} $
} \\
where $\min_\prec[x]$ is the minimum variable appearing in $[x]$ over the lexicographic order $\prec$.

\begin{definition}
	A {\em kernelized conjunction node} $v$ is a conjunction node consisting of a
	distinguished child, we call the {\em core} child, denoted by $ch_{\mathit{core}}(v)$,  and a set of remaining children which define equivalences, denoted by $Ch_{rem}(v)$,  such that:
	
	\begin{enumerate}[(a),topsep=0mm,parsep=1mm]
		\item Every $w_i \in Ch_{rem}(v)$ describes a literal equivalence, i.e., $w_i = \Angle{x \lequ l}$ and the union of $\vartheta(w_i)$, denoted by $E_v$, represents a set of prime literal equivalences.
		\item For each literal equivalence $x \lequ l \in E_v$, $var(l) \notin \Vars(ch_{\mathit{core}}(v))$.
	\end{enumerate}
\end{definition}

We use  $\land_d$ and $\land_k$ to denote decomposed and kernelized conjunctions respectively. 
A \emph{constrained Conjunction \& Decision Diagram} (CCDD) consists of decision nodes, conjunction nodes, and leaf nodes where each decision variable appears at most once on each path from the root to a leaf, and each conjunction node is either decomposed or kernelized.
Figure \ref{fig:full-CCDD} depicts a CCDD. Lai et al. \shortcite{Lai:etal:21} showed that CCDD supports model counting in linear time.

\begin{figure}[ht]
	\centering
	\includegraphics[width = 0.25\textwidth]{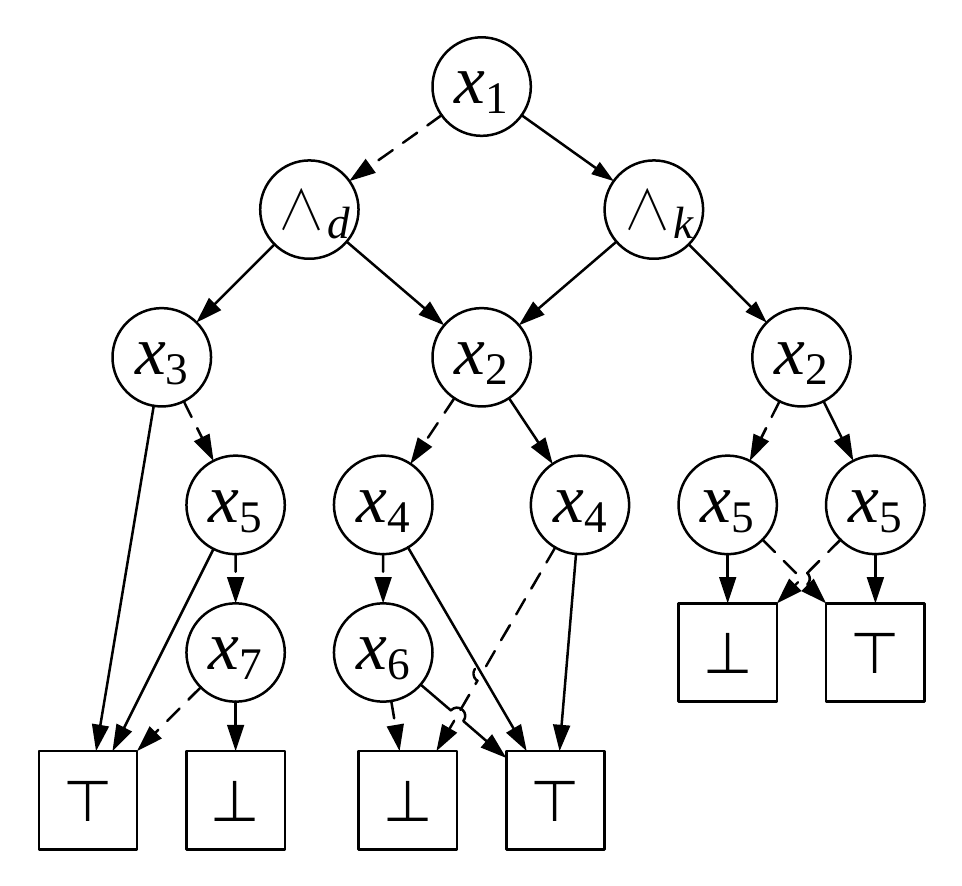}
	\caption{An illustrated CCDD}\label{fig:full-CCDD}
\end{figure}

\section{Related Work}
\label{sec:related}

The related work can be viewed along two directions:
\begin{inparaenum}[(1)]
	\item work related to importance sampling for graphical models; and
	\item work related to approximate compilation for propositional formula.
\end{inparaenum}
While this work focuses on anytime approximate model counting, 
we highlight a line of work in the first category, namely,
the design of efficient hashing-based approximate model counters that seek to provide 
long line of work in the design of efficient hashing-based approximate model counters that seek to provide 
$(\varepsilon,\delta)$-guarantees~\cite{Stockmeyer:83,Gomes:etal:06,ApproxMC,ApproxMC2,ApproxMC3,ApproxMC4}. 

The most related work in the first direction is SampleSearch~\cite{Gogate:Dechter:11,Gogate:Dechter:12}.
For many KC (knowledge compilation) languages, each model of a formula can be seen a particle of the corresponding compiled result.
SampleSearch used the generated models (i.e., samples) to construct an AND/OR sample graph, which can be used to estimate the model count of a CNF formula.
Each AND/OR sample graph can be treated as a partial compiled result in AOBDD (binary version of AND/OR Multi-Valued Decision Diagram \cite{Mateescu:etal:08}).
They showed that the estimate variance of the partial AOBDD is smaller than that of the mean of samples.
Our \PartialKC{} approach has three main differences from SampleSearch.
First, the SampleSearch approach envisages an independent generation of each sample, while the KC technologies used in \PartialKC{} can accelerate the sampling (and thus the convergence), which we experimentally verified.
Second, the decomposition used by the partial AOBDD in SampleSearch
is static, while the one in \PartialKC{} is dynamic.
Results from the KC literature generally suggest that dynamic decomposition is 
more effective than a static one \cite{Dsharp,D4}.
Finally, our KC approach allows to determine if convergence is reached;
but SampleSearch does not.

The related work in the second direction is referred to as approximate KC.
Given a propositional formula and a range of weighting functions, Venturini and Provan \cite{Venturini:Provan:08} proposed two incremental approximate algorithms respectively for prime implicants and DNNF, which selectively compile all solutions exceeding a particular threshold.
Their empirical analysis showed that these algorithms enable space reductions of several orders-of-magnitude over the full compilation.
Intrinsically, partial KC is a type of approximate KC that still admits exactly reasoning to some extent (see Proposition \ref{prop:exact-bound}). 
The output of \PartialKC{} can be used to compute an unbiased estimate of model count, while the approximate \DNNF{} compilation algorithm from Venturini and Provan can only compute a lower bound of the model count.
Some bottom-up compilation algorithms of \OBDD{} \cite{Bryant:86} and \SDD{} \cite{Darwiche:11} also perform compilation in an incremental fashion by using the operator APPLY. However, the OBDD and SDD packages \cite{CUDD,SDD-package} do not overcome the size explosion problem of full KC, because the sizes of intermediate results can be significantly larger than the final compilation result \cite{Huang:Darwiche:04}. Thus, Friedman and Van den Broeck \shortcite{Friedman:VdB:18} proposed an approximate inference algorithm, collapsed compilation, for discrete probabilistic graphical models. 
The differences between \PartialKC{} and collapsed compilation are as follows:
(i) collapsed compilation works in a bottom-up fashion while \PartialKC{} works top-down; (ii)
collapsed compilation is asymptotically unbiased while \PartialKC{} is unbiased;
and (iii) 
collapsed compilation does not support model counting so far.

\section{Partial CCDD}
\label{sec:PCCDD}

In this section, we will define a new representation called \emph{partial} CCDD, used for approximate model counting. For convenience, we call the standard CCDD \emph{full}.

\begin{definition}[Partial CCDD]\label{def:Partial-CCDD}
	Partial CCDD is a generalization of full CCDD, adding two new types of leaf vertices labeled with `$?$' or a number, which are the {\em unknown} and 
{\em known} nodes, respectively. Each arc from a decision node $v$ is labeled by a pair $\Angle{p_b(v), f_b(v)}$ of estimated marginal probability and visit frequency, where $b = 0$ (resp. 1) means the arc is dashed (resp. solid). For a decision node $v$, $p_0(v) + p_1(v) = 1$; $p_c(v) = 0$ iff $sym(ch_b(v)) = \bot$; and $f_b(v) = 0$ iff $sym(ch_b(v)) = ?$.
\end{definition}

Hereafter we use $\Angle{?}$ to denote an unknown node.
For convenience, we require that each conjunctive node cannot have any unknown child.
For simplicity, we sometimes use $f(w)$ and $p(w)$ to denote $\Angle{f_0(w), f_1(w)}$ and $\Angle{p_0(w), p_1(w)}$, respectively, for each decision node $w$.
For a partial CCDD node $v$, we denote the DAG rooted at $v$ by $\D_v$.
We now establish a part-whole relationship between partial and full CCDDs:

\begin{definition}\label{def:part}
	Let $u$ and $u'$ be partial and full CCDD nodes, respectively, from the same formula. $\D_u$ is a \emph{part} of $\D_{u'}$ iff $u$ is an unknown node, or the following conditions hold:
	\begin{compactenum}[(a)]
		\item If $u' = \Angle{\bot}$ (resp. $\Angle{\top}$), then $u = \Angle{\bot}$ (resp. $\Angle{\top}$);
		\item If $u'$ is a known node, then the model count of $u$ equals the number labeled on $u'$;
		\item If $u' = \Angle{x, ch_0(u), ch_1(u)}$, then $sym(u) = x$, and the partial CCDD rooted at $ch_0(u)$ and $ch_1(u)$ are parts of the full CCDD rooted at $ch_0(u')$ and $ch_1(u')$, respectively;
		\item If $u' = \Angle{\land_d, Ch(u')}$, then $sym(u) = \land_d$, $|Ch(u)| = |Ch(u')|$, and each partial CCDD rooted at some child of $u$ is exactly a part of one full CCDD rooted at some child of $u'$; and
		\item If $u' = \Angle{\land_k, ch_{core}(u'), Ch_{rem}(u')}$, then $sym(u) = \land_k$, $Ch_{rem}(u) = Ch_{rem}(u')$, and the partial CCDD rooted at $ch_{core}(u')$ is exactly a part of the full CCDD rooted at $ch_{core}(u)$.
	\end{compactenum}
\end{definition}

Figure \ref{fig:Partial-CCDD} shows two different partial CCDDs from
the full CCDD in Figure \ref{fig:full-CCDD} which can be generated by
\MicroKC{} given later in Algorithm \ref{alg:MicroKC}.
Given a partial CCDD rooted at $u$ that is a part of full CCDD rooted at $u'$, the above definition establishes a mapping from the nodes of $\D_u$ to those of $\D_{u'}$.

\begin{figure}[ht]
	\centering
	\subfloat[]{\label{fig:Partial-CCDD:a}
		\centering
		\begin{minipage}[c]{0.42\linewidth}
			\centering
			\includegraphics[width = \textwidth]{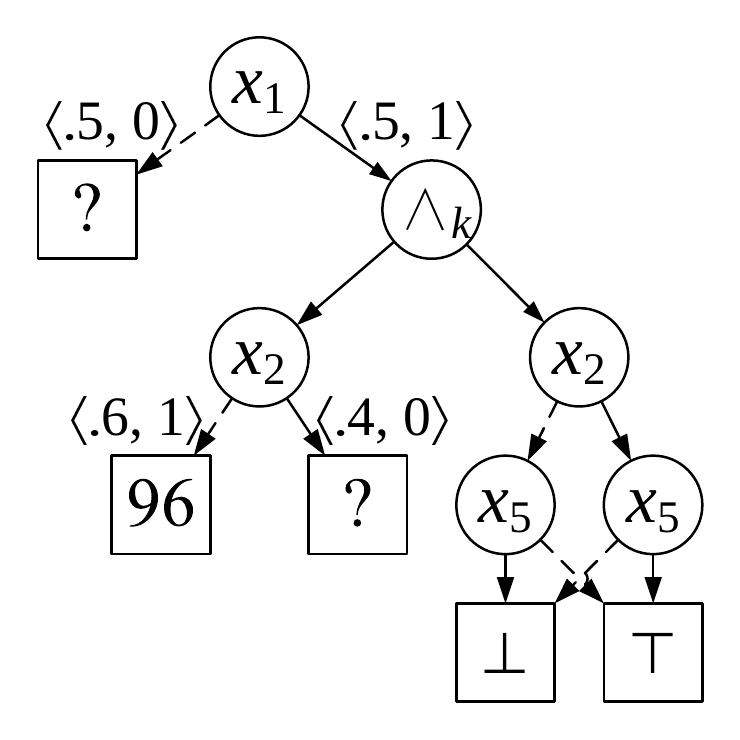}
		\end{minipage}
	}
	\subfloat[]{\label{fig:Partial-CCDD:b}
		\centering
		\begin{minipage}[c]{0.53\linewidth}
			\centering
			\includegraphics[width = \textwidth]{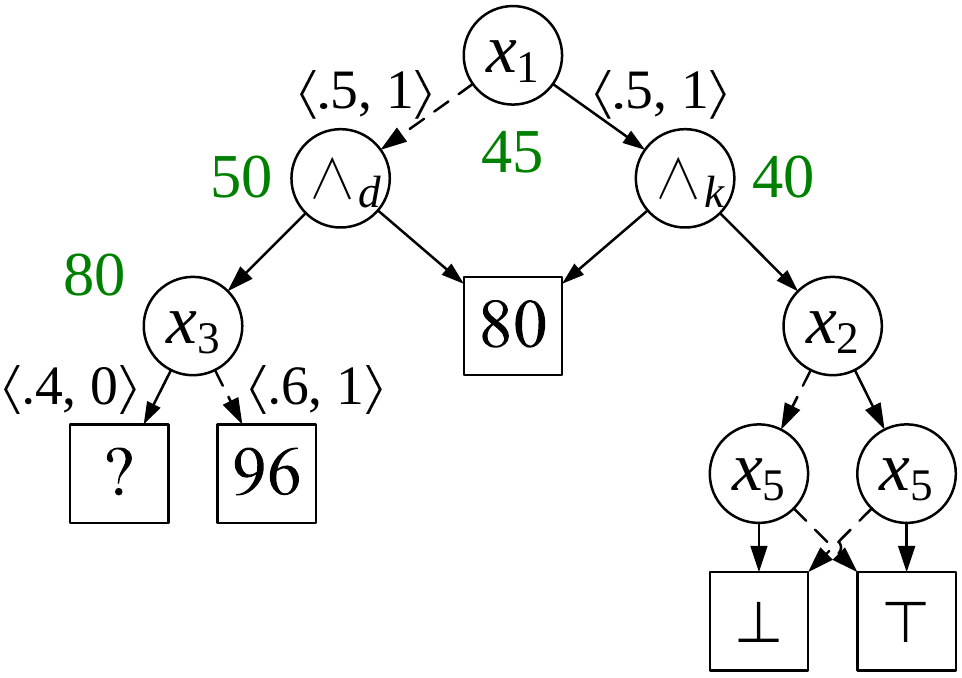}
		\end{minipage}
	}
	\caption{two partial CCDDs}\label{fig:Partial-CCDD}
\end{figure}

A full CCDD can be seen as a compact representation for the model set corresponding to the original knowledge base.
We can use a part of the full CCDD to estimate its model count.
Firstly, a partial CCDD can be used to compute deterministic lower and upper bounds of the model count, respectively:

\begin{restatable}{proposition}{PropBound}\label{prop:exact-bound}
	Let $u$ and $u'$ be, respectively, a partial CCDD node and a full CCDD node over $X$ such that $\D_u$ is a part of $\D_{u'}$. 
	For each unknown node $v$ in $\D_u$ corresponding to $v'$ in $\D_{u'}$ under the part-whole mapping, we assume that we have computed an estimate $\widetilde{Z}(v)$ that is a lower (resp. upper) bound $Z(v')$. A lower (resp. upper) bound of $Z(u')$ can be recursively computed in linear time:
	\begin{equation}\label{eq:bound}
	\widetilde{Z}(u) = \begin{cases}
	0 & sym(u) = \bot \\
	2^{|X|} & sym(u) = \top \\
	sym(u) & sym(u) \in \mathbb{N}  \\
	c^{-1} \cdot \prod_{v \in Ch(u)}\widetilde{Z}(v) & sym(u) = \land_d  \\
	\dfrac{\widetilde{Z}(ch_{\mathit{core}}(u))}{2^{|Ch(u)| - 1}} & sym(u) = \land_k  \\
	\dfrac{\widetilde{Z}(ch_0(u)) + \widetilde{Z}(ch_1(u))}{2} & {sym(u) \in X}
	\end{cases}
	\end{equation}
	where $c = 2^{(|Ch(u)| - 1) \cdot |X|}$.
\end{restatable}

We remark that we must compute lower or upper bound for each unknown node before applying Eq. \eqref{eq:bound}.
In practice, for example, we can compute lower and upper bounds of the model count of an unknown node as $0$ and $2 ^{|X|}$, respectively.
However, we mainly aim at computing an unbiased estimate of the model count.
We will use a randomly generated partial CCDD to compute an unbiased estimate of the model count of the corresponding full CCDD.
The main difference between the new computation and the one in Eq. \eqref{eq:bound} is at the estimation on decision nodes.
Given a randomly generated partial CCDD rooted at $u$, the new estimate of the model count can be computed recursively in linear time:
 \begin{equation}
 	\widehat{Z}(u) = \begin{cases}\label{eq:unbiased}
 	0 & sym(u) = \bot \\
 	2^{|X|} & sym(u) = \top \\
 	sym(u) & sym(u) \in \mathbb{N}  \\
 	c^{-1} \cdot \prod_{v \in Ch(u)}\widehat{Z}(v) & sym(u) = \land_d  \\
 	\dfrac{\widehat{Z}(ch_{\mathit{core}}(u))}{2^{|Ch(u)| - 1}} & sym(u) = \land_k  \\
 	\begin{gathered}\dfrac{\widehat{Z}(ch_0(u)) \cdot f_0(u)}{2p_0(u) \cdot (f_0(u) + f_1(u))} +  \\
 	\dfrac{\widehat{Z}(ch_1(u))\cdot f_1(u)}{2p_1(u) \cdot (f_0(u) + f_1(u))}
 	\end{gathered} & sym(u) \in X
 	\end{cases}
 \end{equation}
where $c = 2^{(|Ch(u)| - 1) \cdot |X|}$. We remark that for each decision node $u$ with one unknown child $v$, the visit frequency $f_b(u)$ on the edge from $u$ to $v$ is 0. Thus, $\widehat{Z}(v) \cdot f_b(u)$ always equals zero in the decision case of Eq. \eqref{eq:unbiased}. 

\begin{example}
Consider the partial CCDD in Figure \ref{fig:Partial-CCDD:b}.
We denote the root by $u$ and the decision child of $ch_0(u)$ by $v$.
\[
\begin{gathered}
\widehat{Z}(v) = \frac{? \times 0}{2 \times 0.4 \times 1} + \frac{96  \times 1}{2 \times 0.6 \times 1} = 80; \\
\widehat{Z}(ch_0(u)) = 2^{ - 1 \times 7}  \times \widehat{Z}(v) \times 80  = 50;
\\
\widehat{Z}(ch_1(u)) = \frac{\widehat{Z}(ch_{core}(ch_0(u)))}{2^1} = 40;
\\
\widehat{Z}(u) = \frac{\widehat{Z}(ch_0(u))  \times 1}{2 \times 0.5 \times 2} + \frac{\widehat{Z}(ch_1(u))  \times 1}{2 \times 0.5 \times 2} = 45.
\end{gathered}
\]
\end{example}

\section{{\PartialKC}: An Anytime Model Counter}
\label{sec:PartialKC}

We aim to estimate model counts for CNF formulas that cannot be solved within time and space limits for exact model counters.
Our approach is to directly generate a randomly partial CCDD formula from the CNF formula rather than from an equivalent full CCDD.
This is achieved with \PartialKC{} given in 
Algorithm \ref{alg:PartialKC}, which compiles a CNF formula into a partial CCDD.

\PartialKC{} calls \MicroKC{} in Algorithm \ref{alg:MicroKC} multiple times in a given timeout $t$.
We use a hash table called $\Cache$ to store the current compiled result implicitly.
Each call of \MicroKC{} updates $\Cache$, implicitly enlarging the current compiled result rooted at $\Cache(\varphi)$. \PartialKC{} reaches convergence in line \ref{line:PartialKC:known} when the root of the resulting partial CCDD making the count exact.
In lines \ref{line:PartialKC:restart-begin}--\ref{line:PartialKC:restart-end}, we will restart the compilation if we encounter memory limits.
Thus, \PartialKC{} is an anytime algorithm.
We assume that the execution of PartialKC is memory-out $k$ times, so PartialKC will generate $k+1$ Partial CCDDs with roots $v_0, \ldots, v_k$. Let $\widehat{Z}_0, \ldots, \widehat{Z}_k$ be the counts marked on the roots.
We also assume that we call MicroKC $N_i$ times in PartialKC for generating the Partial CCDD rooted at $v_i$. Then $\frac{(N_0 \times Z_0) + \cdots + (N_k \times Z_k)}{N_0 + \cdots + N_k}$ is a proper estimate of the true count. Then in line 7 in Algorithm \ref{alg:PartialKC}, $N = N_0 + \cdots + N_i, M = N_0 + \cdots + N_{i-1}$, and $\widehat{Z} = \frac{(N_0 \times \widehat{Z}_0) + \cdots + (N_k \times \widehat{Z}_{k-1})}{M}$. After line 8, $\widehat{Z} = \frac{(N_0 \times Z_0) + \cdots + (N_k \times \widehat{Z}_k)}{N}$.

\begin{algorithm}[!htb]
	\small
	\caption{\PartialKC($\varphi$, $t$)} \label{alg:PartialKC}
	$N \leftarrow M \leftarrow Z \leftarrow 0$\;
	\While {running time does not exceed to $t$}{
		$N \leftarrow N+1$\;
		$\MicroKC(\varphi)$\;
		\lIf {$\Cache(\varphi)$ is known} {\KwRet $sym(\Cache(\varphi))$}\label{line:PartialKC:known}
		\If {exceeding to memory} {\label{line:PartialKC:restart-begin}
			$v \leftarrow \Cache(\varphi)$\;
			$Z = \frac{M \cdot Z + (N-M) \cdot \widehat{Z}(v)}{N}$\;
			$M \leftarrow N$\;
			Clear \Cache{}\;
		}\label{line:PartialKC:restart-end}
	}
	$v \leftarrow \Cache(\varphi)$\;
	\KwRet $\frac{M \cdot Z + (N-M) \cdot \widehat{Z}(v)}{N}$
\end{algorithm}

We estimate the hardness of the input CNF formula in line \ref{line:MicroKC:easy}, and if it is easy, we will obtain a known node by calling an exact model counter,  ExactMC \cite{Lai:etal:21} which uses a full \CCDD.
In lines \ref{line:MicroKC:initial-begin}--\ref{line:MicroKC:initial-end}, we deal with the case of the initial call of \MicroKC{} on $\varphi$.
We try to kernelize in lines \ref{line:MicroKC:initial-kernel-begin}--\ref{line:MicroKC:initial-kernel-end}.
Otherwise, we decompose $\varphi$ into a set of sub-formulas without shared variables in line \ref{line:MicroKC:decompose}.
In lines \ref{line:MicroKC:initial-decomposition-begin}--\ref{line:MicroKC:initial-decomposition-end}, we deal with the case where $\varphi$ is decomposable, and call \MicroKC{} recursively for each sub-formula.

Otherwise, we deal with the case where $\varphi$ is not decomposable in lines \ref{line:MicroKC:initial-decision-begin}--\ref{line:MicroKC:initial-decision-end}.
We introduce a decision node $u$ labeled with a variable $x$ from $\Vars(\varphi)$.
We estimate the marginal probability of $\varphi$ over $x$ in line \ref{line:MicroKC:MargProb} and sample a Boolean value $b$ with this probability in the next line.
We remark that the variance of our model counting method depends on the accuracy of the marginal probability estimate, discussed further in Section \ref{sec:imp}.
Then we generate the children of $u$, updating the probability and frequency.

We deal with the case of repeated calls of \MicroKC{} on $\varphi$ in lines \ref{line:MicroKC:repeat-begin}--\ref{line:MicroKC:repeat-end}.
The inverse function $\Cache^{-1}(v)$ of $\Cache$ is used for finding formula $\varphi$ such that $\Cache(\varphi) = v$.
\begin{algorithm}[!htb]
	\small
	\caption{\MicroKC($\varphi$)} \label{alg:MicroKC}
	\lIf {$\varphi = \mathit{false}$} {\KwRet $\Angle{\bot}$} \label{line:MicroKC:base-case-begin}
	\lIf {$\varphi = \mathit{true}$} {\KwRet $\Angle{\top}$} \label{line:MicroKC:base-case-end}
	\lIf {$\EasyInstance(\varphi)$} {\KwRet $\Angle{\text{ExactMC}(\varphi)}$}\label{line:MicroKC:easy}
	\If {$\Cache(\varphi) = nil$} {\label{line:MicroKC:initial-begin}
		\If {$\Kernelizable(\varphi)$} {\label{line:MicroKC:initial-kernel-begin}
		$E \leftarrow \DetectLitEqu(\varphi)$\;\label{line:Panini:liteq}
		\If {$|\Floor{E}| > 0$}{ 
			$\psi \leftarrow \ConstructCore(\varphi, \Floor{E})$\; \label{line:Panini:substitute}
			$v \leftarrow \MicroKC(\psi)$\; \label{line:MicroKC:kernel}
			$V \leftarrow \{\Angle{x \lequ l} \mid x \lequ l \in \Floor{E}\}$\;
			\KwRet $Cache(\varphi) \leftarrow \Angle{\land_k, \{v\} \cup V}$\;\label{line:MicroKC:kernelized-node}
		}
		}\label{line:MicroKC:initial-kernel-end}
		$\Psi \leftarrow \Decompose(\varphi)$\; \label{line:MicroKC:decompose}
		\uIf {$|\Psi| > 1$} {
			$V \leftarrow \{\MicroKC(\psi ) \mid \psi \in \Psi\}$\;\label{line:MicroKC:initial-decomposition-begin}
			\KwRet $\mathit{Cache}(\varphi) \leftarrow \Angle{\land_d, V}$\label{line:MicroKC:initial-decomposition-end}
		}
		\Else{ 
			$x \leftarrow \PickGoodVar(\varphi)$\; \label{line:MicroKC:initial-decision-begin}
			$p \leftarrow \MargProb(\varphi, x)$\;\label{line:MicroKC:MargProb}
			$b \sim \mathit{Bernoulli}(p)$\;\label{line:MicroKC:initial-sample}
			Create a decision node $u$ with $sym(u) = x$ \;
			$ch_b(u) \leftarrow \MicroKC(\varphi[x \mapsto b])$\;
			$ch_{1-b}(u) \leftarrow \Angle{?}$\;
			$p_{0}(u) \leftarrow 1-p$; $p_{1}(u) \leftarrow p$\;
			$f_{b}(u) \leftarrow 1$; $f_{1-b}(u) \leftarrow 0$\;
			\KwRet $Cache(\varphi) \leftarrow u$\label{line:MicroKC:initial-decision-end}
		}
	}\label{line:MicroKC:initial-end}
	$v \leftarrow \Cache(\varphi)$\;\label{line:MicroKC:repeat-begin}
	\uIf {$v$ has no unknown descendants} {
		Let $c$ be the model count of $v$\;\label{line:Micro:repeat-known-begin}
		\KwRet $Cache(\varphi) \leftarrow \Angle{c}$\;\label{line:Micro:repeat-known-end}
	}
 	\uElseIf {$v$ is kernelized} {\label{line:Panini:kernel-begin}
			$v' \leftarrow \MicroKC(\Cache^{-1}(ch_{core}(v)))$\; \label{line:Panini:kernel}
			\KwRet $Cache(\varphi) \leftarrow \Angle{\land_k, \{v'\} \cup Ch_{rem}(v)}$\;\label{line:Micro:kernelized-node}
	}\label{line:Panini:kernel-end}
	\uElseIf {$v$ is decomposed} {
		$V \leftarrow \{\MicroKC(\Cache^{-1}(w)): w \in Ch(v)\}$\;
		\KwRet $\mathit{Cache}(\varphi) \leftarrow \Angle{\land_d, V}$\label{line:Panini:decomposition}
	}
	\Else{ 
		$b \sim \mathit{Bernoulli}(p_1(v))$\;\label{line:MicroKC:repeat-sample}
		$f_b(v) \leftarrow f_b(v) + 1$\;
		$ch_b(v) \leftarrow \MicroKC(\varphi[sym(v) \mapsto b])$\;\label{line:MicroKC:decision-end}
	}\label{line:MicroKC:repeat-end}
\end{algorithm}

\begin{example}\label{exam:PartialKC}
	We run \PartialKC{} on the formula $\varphi = (x_1 \lor x_3 \lor x_5 \lor \lnot x_7 ) \land ( x_4 \lor x_6 ) \land (\lnot x_2  \lor x_4) \land (\lnot x_1 \lor \lnot x_2 \lor x_5 ) \land (\lnot x_1 \lor  x_2  \lor \lnot x_5 )$.
	For simplicity, we assume that \PickGoodVar{} chooses the variable with the smallest subscript and \EasyInstance{} returns true when the formula has less than three variables.
	For the first calling of \MicroKC($\varphi$), the condition in line \ref{line:MicroKC:initial-begin} is satisfied. 
	We assume that the marginal probability of $\varphi$ over $x_1$ is estimated as 0.5 and 1 is sampled in line \ref{line:MicroKC:initial-sample}.
	Then \MicroKC($\varphi_1$) is recursively called, where $\varphi_1 = ( x_4 \lor x_6 ) \land (\lnot x_2  \lor x_4) \land (\lnot x_2 \lor x_5 ) \land (x_2  \lor \lnot x_5)$.
	We kernelize $\varphi_1$ as $\varphi_2 = ( x_4 \lor x_6 ) \land (\lnot x_2  \lor x_4)$ and then invoke \MicroKC($\varphi_2$). 
	Similarly, the condition in \ref{line:MicroKC:initial-begin} is satisfied. 
	We assume that the estimated marginal probability of $\varphi_2$ over $x_2$ is 0.4 and 0 is sampled in line \ref{line:MicroKC:initial-sample}.
	Then \MicroKC($\varphi_{3}$) is recursively called, where $\varphi_{3} = x_4 \lor x_6$, and we call ExactMC($\varphi_{3}$) to get a count 96. 
	Finally, the partial CCDD in Figure \ref{fig:Partial-CCDD:a} is returned.
	For the second calling of \MicroKC($\varphi$), the condition in line \ref{line:MicroKC:initial-begin} is not satisfied. 
	We get the stored marginal probability of $\varphi$ over $x_1$ and assume that 0 is sampled in line \ref{line:MicroKC:repeat-sample}.
	Then \MicroKC{} is recursively called on $\varphi_4 = (x_3 \lor x_5 \lor \lnot x_7 ) \land ( x_4 \lor x_6 ) \land (\lnot x_2  \lor x_4)$, and then \MicroKC{} is recursively called on $\varphi_5 = x_3 \lor x_5 \lor \lnot x_7$ and $ \varphi_2 = ( x_4 \lor x_6 ) \land (\lnot x_2  \lor x_4)$.
	Finally, we generate the partial CCDD in Figure \ref{fig:Partial-CCDD:b}.
\end{example}

\begin{restatable}{proposition}{PropPartialKC}\label{prop:PartialKC}
	Given a CNF formula $\varphi$ and a timeout setting $t$, the output of \PartialKC($\varphi$, $t$) is an unbiased estimate of the model count of $\varphi$.
\end{restatable}
\begin{proof}
	If \PartialKC{} exceeds memory limits, it just restarts and finally returns the average estimate. Thus, we just need to show when \PartialKC{} does not exceed memory, it outputs an unbiased estimate of the true count. 
	
	Given the input $\varphi$, we denote all of the inputs of recursively calls of \MicroKC{} as a sequence $S = (\varphi_1, \ldots, \varphi_n=\varphi)$ in a bottom-up way in the call of \PartialKC($\varphi$, $t$).
	Let $N$ be the final number of calls of \MicroKC{} on $\varphi$, and let $Z(\varphi_i)$ be the true model count $\varphi_i$.
	
	We first prove the case $N=1$ by induction with the hypothesis that the call of \MicroKC($\psi$) returns an unbiased estimate of the model count of $\psi$ if $|\Vars(\psi)| < |\Vars(\varphi_n)|$. Thus, the calls of \MicroKC{} on $\varphi_1, \ldots, \varphi_{n-1}$ return unbiased estimates of the model counts, and the results are stored in $Cache$.
	We denote the output of \MicroKC($\varphi_i$) by $u_i$.
	The cases when \MicroKC{} returns a leaf node or $\Cache(\varphi_n) = nil$ in line 4 is obvious.
	We proceed by a case analysis (the equations about true counts can be found in the proofs of Propositions 1--2 in \cite{Lai:etal:21}):
	\begin{enumerate}[(a)]
		\item $E \not= \emptyset$ in line 6: The input of the recursive call is $\varphi_{n-1}$. According to the induction hypothesis, $\expectation{}{\widehat{Z}(v)} = Z(\varphi_{n-1})$. Thus, $\expectation{}{\widehat{Z}(u_n)} =  \dfrac{Z(\varphi_{n-1})}{2^{|Ch(u)| - 1}} = Z(\varphi_{n})$ .
		\item $|\Psi| = m > 1$ in line 15: The input of the recursive calls are $\varphi_{n-m}, \ldots, \varphi_{n-1}$. According to the induction hypothesis, $\expectation{}{\widehat{Z}(u_i)} = Z(\varphi_{i})$ ($n-m \le i \le n-1$). Due to the conditional independence, $\expectation{}{\widehat{Z}(u_n)} =  c^{-1} \cdot \prod_{i=n-m}^{n-1}\expectation{}{\widehat{Z}(u_i)} = c^{-1} \cdot \prod_{i=n-m}^{n-1}Z(\varphi_i) = Z(\varphi_{n})$.
		\item $|\Psi| = 1$ in line 15. The input of the recursive call is $\varphi_{n-2} = \varphi_n[x \mapsto false]$ and $\varphi_{n-1} = \varphi_n[x \mapsto true]$. Thus, $Z(\varphi_{n}) = \frac{1}{2} \cdot (Z(\varphi_{n-2}) + Z(\varphi_{n-1}))$. 
		\begin{align*}
			\expectation{}{\widehat{Z}(u)}&=\frac{\expectation{}{\widehat{Z}(ch_0(u))} \cdot f_0(u)}{2p_0(u) \cdot (f_0(u) + f_1(u))} \cdot p_0(u)+\\
			&\mathrel{\phantom{=}}\frac{\expectation{}{\widehat{Z}(ch_1(u))}\cdot f_1(u)}{2p_1(u) \cdot (f_0(u) + f_1(u))} \cdot p_1(u)\\
			&=\frac{1}{2} \cdot \left(\expectation{}{\widehat{Z}(ch_0(u))} + \expectation{}{\widehat{Z}(ch_1(u))}\right)\\
			&=Z(\varphi_{n}).
		\end{align*}
	\end{enumerate}
	
	We can also prove the case $N > 1$ by induction.
	We call \MicroKC{} on the same formula at most $N$ times; in other words, each formula in $S$ at most $N$ times.
	It is assumed that the $M$-th call of \MicroKC($\psi$) returns an unbiased estimate of the model count of $\psi$ if $|\Vars(\psi)| < |\Vars(\varphi_n)|$ or $M < N$. 
	Then the proof for this case is similar to the case $N=1$ except the decision case. 
	We remark that for a decision node without any unknown node, we can compute its exact count by applying Eq. \eqref{eq:bound}.
	When it returns a known node in lines \ref{line:Micro:repeat-known-begin}--\ref{line:Micro:repeat-known-end}, we can get the exact count.
	Now we prove the case when a non-unknown node is returned.
	For convenience, we denote the returned node by $u$.
	The value of $f(u)$ is independent from what are the children of $u$. 
	We have $\expectation{}{\frac{f_0(u)}{N}}=p_0(u)$ and $\expectation{}{\frac{f_1(u)}{N}}=p_1(u)$. From the decision case of Eq. \eqref{eq:unbiased}, we get the following equation:
	\begin{align*}
		\expectation{}{\widehat{Z}(u)}&=\frac{\expectation{}{\widehat{Z}(ch_0(u))} \cdot \expectation{}{f_0(u)}}{2p_0(u) \cdot N} +\\
		&\mathrel{\phantom{=}}\frac{\expectation{}{\widehat{Z}(ch_1(u))}\cdot \expectation{}{f_1(v)}}{2p_1(u) \cdot N}+\\
		&=\frac{1}{2} \cdot \left(\expectation{}{\widehat{Z}(ch_0(u))} + \expectation{}{\widehat{Z}(ch_1(u))}\right)\\
		&=Z(\varphi_{n}).
	\end{align*}
\end{proof}

As we get an unbiased estimate of the exact count,
probabilistic lower bounds can be obtained by Markov's inequality \cite{Wei:Selman:05,Gomes:etal:07}.

\subsection{Implementation}
\label{sec:imp}

We implemented {\PartialKC} in the toolbox KCBox.\footnote{KCBox is available at: \url{https://github.com/meelgroup/KCBox}}
For the details of functions \Decompose{}, \PickGoodVar{}, \Kernelizable{}, \ConstructCore{} and ExactMC, we refer the reader to ExactMC \cite{Lai:etal:21}.
In the function \EasyInstance{}, we rely on the number of variables as a proxy for the hardness of a formula, in particular at each level of recursion, we classify a formula $\varphi$ to be easy if  $|\Vars(\varphi)| \le \mathit{easy\_bound}$. We define $\mathit{easy\_bound}$ as the minimum of 512 and $\mathit{easy\_param}$:
$$\mathit{easy\_param} = \begin{cases}
\frac{3}{4} \cdot \#\mathit{NonUnitVars} & \mathit{width} \le 32\\
\frac{2}{3} \cdot \#\mathit{NonUnitVars} & 32 < \mathit{width} \le 64\\
\frac{1}{2} \cdot \#\mathit{NonUnitVars} & \mathit{otherwise}
\end{cases}$$
where $\#\mathit{NonUnitVars}$ is the number of variables appearing in the non-unit clauses of the original formula, and $width$ is the minfill treewidth \cite{book:Darwiche:09}.

\MicroKC{} can be seen as a sampling procedure equipped with KC technologies, and the variance of the estimated count depends on three main factors.
First, the variance depends on the quality of predicting marginal probability.
Second, the variance of a single calling of \MicroKC{} depends on the number of sampling Boolean values in lines \ref{line:MicroKC:initial-sample} and \ref{line:MicroKC:repeat-sample}.
The fewer the samples from the Bernoulli distributions, the smaller the variance.
Finally, the variance depends on the number of \MicroKC{} calls when fixing their total time.
We sketch four KC technologies
for reducing the variance.

The first technology is how to implement \MargProb{}.
Without consideration of the dynamic decomposition, each call of \MicroKC{} can be seen as a process of importance sampling, where the resulting partial CCDD is treated as the proposal distribution.
Similar to importance sampling, it is easy to see that the variance of using \PartialKC{} to estimate model count depends on the quality of estimating the marginal probability in line \ref{line:MicroKC:MargProb}.
If the estimated marginal probability equals the true one, \PartialKC{} will yield an optimal (zero variance) estimate.
In principle, the exact marginal probability can be compute from an equivalent full CCDD.
This full compilation is however impractical computationally.
Rather, \MargProb{} estimates the marginal probability via compiling the formula into full CCDDs on a small set $P$ of projected variables by $\ProjectedKC$ in Algorithm \ref{alg:ProjectedKC}.
In detail, we first perform two projected compilations by calling	$\ProjectedKC(\varphi[x \mapsto false], P)$ and 
$\ProjectedKC(\varphi[x \mapsto true], P)$ with the outputs $u$ and $v$, and then use the compiled results to compute the marginal probability that is equal to $\frac{Z(v)}{Z(u)+Z(v)}$.

\begin{algorithm}[tb]
	\caption{\ProjectedKC($\varphi$, $P$)} \label{alg:ProjectedKC}
	\lIf {$\varphi = \mathit{false}$} {\KwRet $\Angle{\bot}$} \label{line:ProjectedKC:base-case-begin}
	\lIf {$\varphi = \mathit{true}$} {\KwRet $\Angle{\top}$} \label{line:ProjectedKC:base-case-end}
	\lIf {$\mathit{ProjCache}(\varphi) \not= nil$} {\KwRet $\mathit{ProjCache}(\varphi)$}
	$\Psi \leftarrow \Decompose(\varphi)$\; \label{line:ProjectedKC:decompose}
	\uIf {$|\Psi| > 1$} {
		$W \leftarrow \{\ProjectedKC(\psi, P) \mid \psi \in \Psi\}$\;
		\KwRet $\mathit{ProjCache}(\varphi) \leftarrow \Angle{\land_d, W}$\label{line:ProjectedKC:decomposition}
	}
	\Else{ 
		\If{ $\Vars(\varphi) \cap P = \emptyset$  }{
			\lIf {$\varphi$ is satisfiable} {\KwRet $\Angle{\top}$}
			\lElse {\KwRet $\Angle{\bot}$} 
		}
		$x \leftarrow \PickGoodVar(\varphi, P)$\; \label{line:ProjectedKC:decision-begin}
		$w_0 \leftarrow \ProjectedKC(\varphi[x \mapsto false], P)$\;
		$w_1 \leftarrow \ProjectedKC(\varphi[x \mapsto true], P)$\; 
		\KwRet $\mathit{ProjCache}(\varphi) \leftarrow \Angle{x, w_0, w_1}$\; \label{line:ProjectedKC:decision-end}
	}
\end{algorithm}

The second technology is dynamic decomposition in line \ref{line:MicroKC:decompose}.
We employ a SAT solver to compute the implied literals of a formula, and use these implied literals to simplify the formula.
Then we decompose the residual formula according to the corresponding primal graph.
We can reduce the variance based on the following:
\begin{inparaenum}[(a)]
	\item the sampling is backtrack-free, this remedies the rejection problem of sampling; 
	\item we reduce the variance by sampling from a subset of the variables, also known as Rao-Blackwellization \cite{Casella:Robert:96}, and its effect is strengthened by decomposition;
	\item after decomposing, more virtual samples can be provided in contrast to the original samples \cite{Gogate:Dechter:12}.
\end{inparaenum}

The third technology is kernelization, which can simplify a CNF formula. After kernelization, we can reduce the number of sampling Boolean values in lines \ref{line:MicroKC:initial-sample} and \ref{line:MicroKC:repeat-sample}.
It can also save time of for computing implied literals in the dynamic decomposition as kernelization often can simplify the formula.

The fourth technology is the component caching implemented using hash table $\Cache$.
In different calls of \MicroKC{}, the same sub-formula may need to be processed several times.
Component caching, can save the time of dynamic decomposition, and accelerate the sampling.
It can also reduce the variance by merging the calls to \MicroKC{} on the same sub-formula.
Consider Example \ref{exam:PartialKC} again. We call \MicroKC{} on $\varphi_{2}$ twice, and obtain a known node.
The corresponding variance is then smaller than that of a single call of \MicroKC{}.
Our implementation uses the component caching scheme in sharpSAT~\cite{sharpSAT}.

\begin{table*}[htb]
	\centering
	\footnotesize
	\aboverulesep = 0.2ex
	\belowrulesep = 0.2ex
	\begin{tabularx}{\linewidth}{l>{\centering\arraybackslash}*{11}{>{\centering\arraybackslash}X}}\toprule
		\multirow{3}*{domain (\#, \#known)} & \multicolumn{3}{c}{exact} & \multicolumn{8}{c}{approximate} \\\cmidrule(rl){2-4}\cmidrule(rl){5-12}
		{} & \multirow{2}*{D4} & \multirow{2}*{\hspace{-0.58cm} sharpSAT-td} & \multirow{2}*{\hspace{-0.05cm}ExactMC} & \multicolumn{3}{c}{\PartialKC{}} & \multicolumn{2}{c}{satss} & \multicolumn{2}{c}{STS} & \multirow{2}*{\hspace{-0.3cm}ApproxMC} \\\cmidrule(rl){5-7}\cmidrule(rl){8-9}\cmidrule(rl){10-11}
		{} & {} & {} & {} & \#approx & \#conv & $\varepsilon=4$ & \#approx & {$\varepsilon=4$} & \#approx & {$\varepsilon=4$} \\\midrule
		Bayesian-Networks (201, 186)  & 179 & 186 & 186 & \textbf{195} & 186 & \textit{186} & 18 & 17 & 157 & 148 & 172 \\
		BlastedSMT (200, 183) & 162 & 163 & 169 & \textbf{200} & 168 & 177 & 150 & 129 & \textbf{200} & 178 & \textit{197} \\
		Circuit (56, 51) & 50 & 50 & 51 & \textbf{54} & 50 & \textit{50} & 15 & 13 & 50 & 44 & 46 \\
		Configuration (35, 35) & 34 & 32 & 31 & \textbf{35} & 29 & \textit{33} & 33 & 31 & 35 & 13 & 15 \\
		Inductive-Inference (41, 23) & 18 & 21 & 22 & 37 & 21 & \textit{23} & 40 & 19 & \textbf{41} & \textit{23} & 21 \\
		Model-Checking (78, 78) & 75 & \textbf{78} & \textbf{78} & \textbf{78} & 77 & \textit{78} & 11 & 7 & 10 & 8 & \textbf{\textit{78}} \\
		Planning (243, 219) & 208 & 215 & 213 & \textbf{240} & 209 & \textit{213} & 169 & 126 & 238 & 103 & 147 \\
		Program-Synthesis (220, 119) & 91 & 78 & 108 & \textbf{135} & 100 & 113 & 81 & 64 & 125 & 108 & \textit{115} \\
		QIF (40, 33) & 29 & 30 & 32 & \textbf{40} & 24 & 32 & 15 & 13 & 27 & 25 & \textbf{\textit{40}} \\
		MC2022\_public (100, 85) & 72 & 76 & 76 & \textbf{89} & 69 & \textit{79} & 55 & 45 & 88 & 77 & 61 \\\midrule[0.02em]
		Total (1214, 1012) & 918 & 929 & 966 & \textbf{1103} & 933 & \textit{984} & 585 & 463 & 971 & 727 & 892 \\\bottomrule
	\end{tabularx}
	\caption{Comparative counting performance between D4, sharpSAT-td, ExactMC, \PartialKC{}, satss, STS, and ApproxMC. Each cell below D4, sharpSAT-td, ExactMC, \#approx of \PartialKC{}, satss, STS, and the ApproxMC column refers to the number of solved instances, and the maximum numbers are marked in bold. \#known denotes the number of instances that solved by D4, sharpSAT-td, or ExactMC with a longer timeout of four hours. \#conv refers to the number of instances where convergence was reached. $\varepsilon=4$ refers to the number of instances where the reported count falls into $[(1 + \varepsilon)^{-1}Z, (1 + \varepsilon)Z]$ with the true count $Z$. We remark that each count reported by ApproxMC falled into this interval. The maximum numbers of estimates falling into the interval are marked in italics.}
	\label{tab:expri:counting}
\end{table*}

\begin{figure*}[tb]
	\centering
	\subfloat[inductive-inference\_ii32c2]{\label{fig:convergence:a}
		\centering
		\begin{minipage}[c]{0.33\linewidth}
			\centering
			\includegraphics[width = \textwidth]{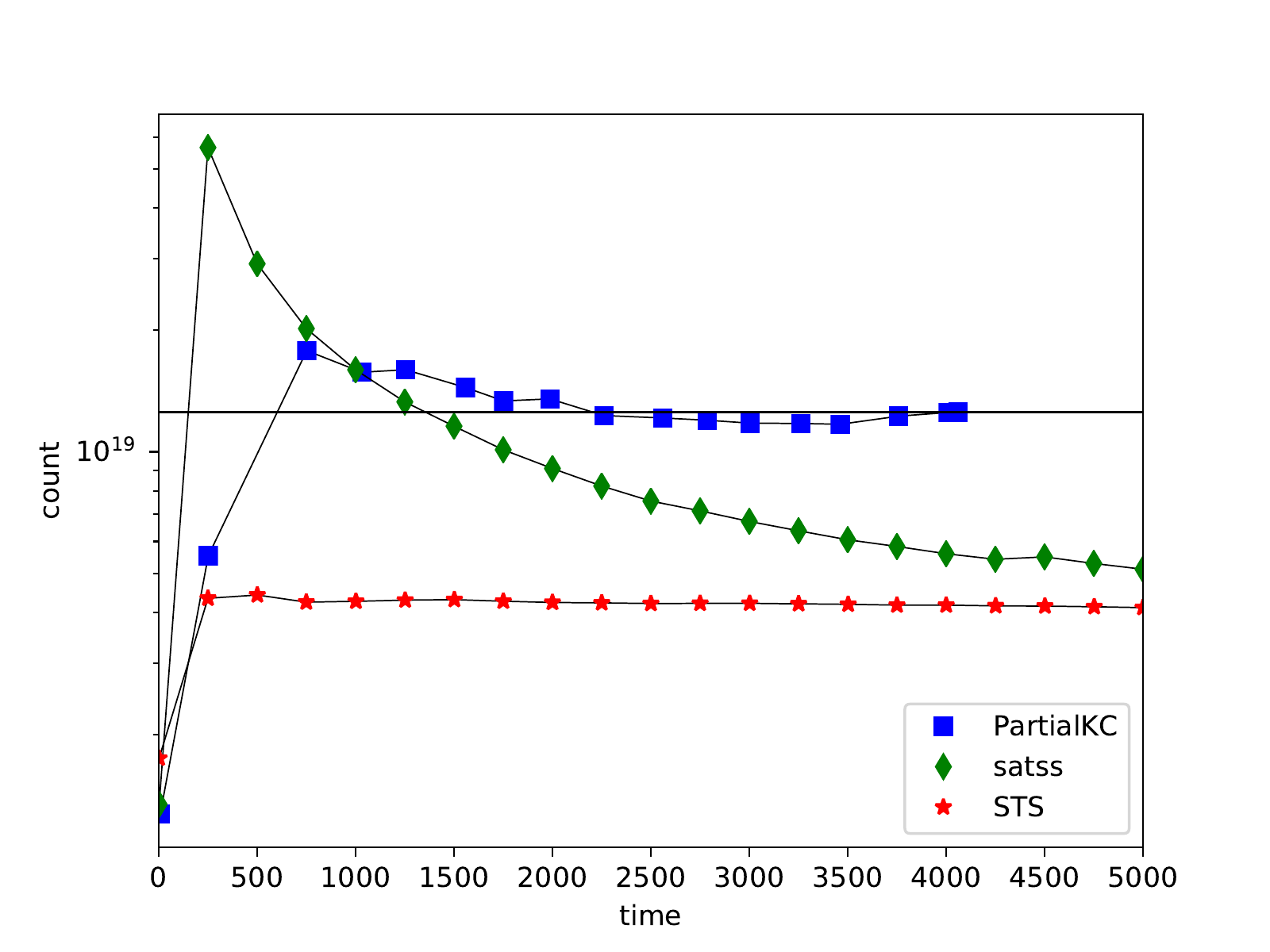}
		\end{minipage}
	}
	\subfloat[ModelChecking\_bmc-ibm-7]{\label{fig:convergence:b}
		\centering
		\begin{minipage}[c]{0.33\linewidth}
			\centering
			\includegraphics[width = \textwidth]{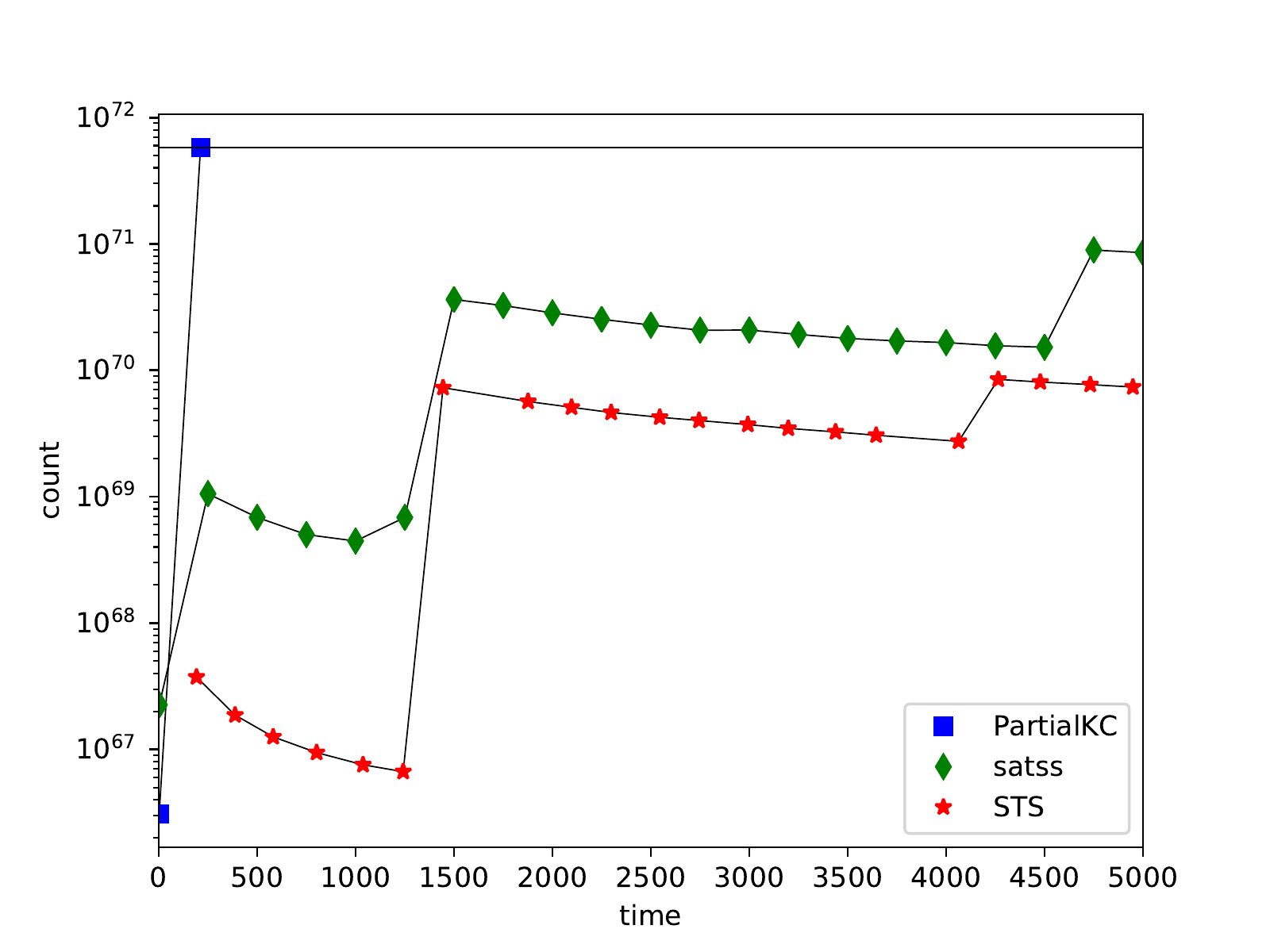}
		\end{minipage}
	}
	\subfloat[Planning\_blocks\_right\_3\_p\_t10]{\label{fig:convergence:c}
		\centering
		\begin{minipage}[c]{0.33\linewidth}
			\centering
			\includegraphics[width = \textwidth]{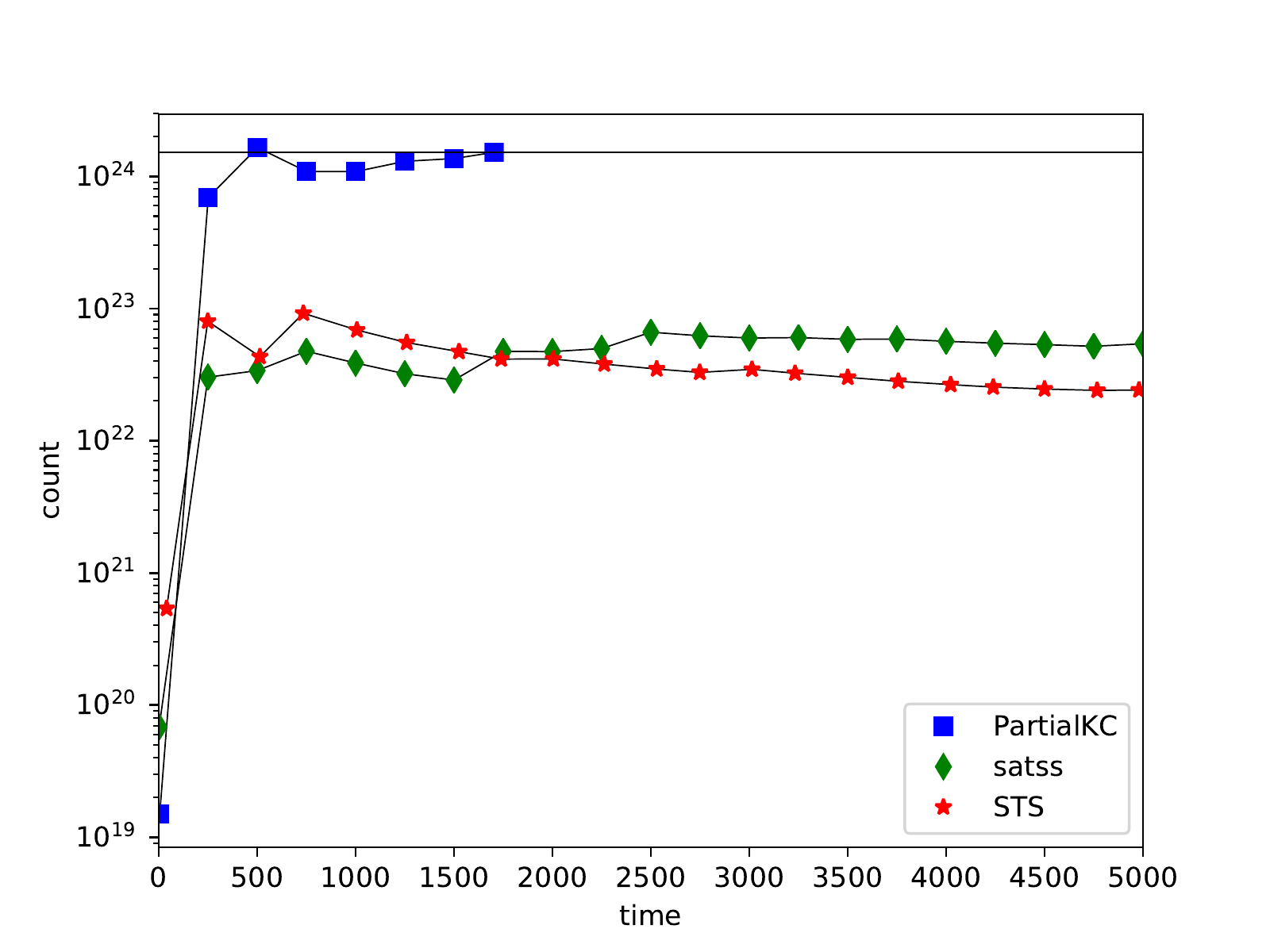}
		\end{minipage}
	}
	\caption{The convergence performance of \PartialKC{}, SampleSearch and STS over three instances with true counts depicted by straight horizontal lines. (Best viewed in color)}\label{fig:convergence}
\end{figure*}
\section{Experiments}
\label{sec:experiments}

We evaluated {\PartialKC}
on a comprehensive set of benchmarks: (i) 1114 benchmarks from a wide range of application areas, including automated planning, Bayesian networks, configuration, combinatorial circuits, inductive inference, model checking, program synthesis, and quantitative information flow (QIF) analysis; and (ii) 100 public instances adopted in the Model Counting Competition 2022. We remark that the 1114 instances have been employed in the past to evaluate model counting and knowledge compilation techniques~\cite{D4,Lai:etal:17,Fremont:etal:17,ApproxMC3,Lai:etal:21}. 
The experiments were run on a cluster (HPC cluster with job queue)
where each node has 2xE5-2690v3 CPUs with 24 cores and 96GB of RAM. 
Each instance was run on a single core with a timeout of 5000 seconds and 8GB memory.

We compared exact counters D4 \cite{D4}, sharpSAT-td \cite{SharpSAT-TD}, and ExactMC \cite{Lai:etal:21}, and approximate counters satss (SampleSearch, \cite{Gogate:Dechter:11}), STS (SearchTreeSampler, \cite{STS}), and the 
latest version (from the developers)
of ApproxMC \cite{ApproxMC,Arjun} that combines the independent support computation technique, Arjun, with ApproxMC. 
We remark that both satss and STS are anytime.
ApproxMC was run with $\varepsilon=4$ and $\delta = 0.2$.
We used the pre-processing tool B+E \cite{BplusE} for all the instances, which was shown very powerful in model counting.
Consistent with recent studies, we excluded the pre-processing time from the solving time for each tool as pre-processed instances were used on all solvers.

Table \ref{tab:expri:counting} shows the performance of the seven counters.
The results show that \PartialKC{} has the best scalability as it can solve approximately many more instances than the other six tools.
We emphasize that there are 123 instances in 1103 that was not solved by D4, sharpSAT-td, and ExactMC.
The results also show that \PartialKC{} can get convergence on 933 instances, i.e., it gets the exact counts for those instances.
It is surprising that \PartialKC{}, due to the anytime and sampling nature of the algorithm which entails some additional costs, can still outperform state-of-the-art exact counters D4 and sharpSAT-td.

We evaluate the accuracy of \PartialKC{} from two aspects.
First, we consider each estimated count that falls into the interval $[(1 + \varepsilon)^{-1}Z, (1 + \varepsilon)Z]$ of the true count $Z$. 
We say an estimate is \emph{qualified} if it falls into this interval.
Note that this is only for the instances where the true count is known.
We choose $\varepsilon=4$. 
If the estimate falls into $[0.2Z, 5Z]$, it is basically the same order of magnitude as the true count.
The results show that \PartialKC{} computed the most qualified estimates.
We highlight that there are 835 instances where PartialKC converged in one minute of CPU time. This number is much greater than the numbers of qualified solved instances of satss and STS.
We remark that convergence is stricter than the requirement that an estimate falls in $[(1 + \varepsilon)^{-1}Z, (1 + \varepsilon)Z]$.
Second, we compare the accuracy of \PartialKC{}, satss, and STS in terms of the average log relative error \cite{Gogate:Dechter:12}. 
Given an estimated count $\widehat{Z}$, the log-relative error is defined as $\left|\frac{\log(Z)-\log(\widehat{Z})}{\log(Z)}\right|$.
For fairness, we only consider the instances that are approximately solved by all of \PartialKC{}, satss, and STS.
Our results show that the average log relative errors of \PartialKC{}, satss, and STS are 0.0075, 0.0081, and 0.0677, respectively; that is, \PartialKC{} has the best accuracy.

The convergence performance of \PartialKC{} shows that it can give very accurate estimates. 
For instance, if we evaluate with $\varepsilon=0.1$ rather than $\varepsilon=4$
in Table \ref{tab:expri:counting}, \PartialKC{} still can get qualified estimates on 940 (44 fewer) instances, while the two other anytime tools satss and STS give qualified estimates on 337 (126 fewer) and 575 (157 fewer) instances, respectively.
We compare the convergence performance of \PartialKC{}, satss, and STS further on three selected instances depicted in Figure \ref{fig:convergence}.
PartialKC gets fast convergence to the exact count, while neither satss nor STS converges.

\section{Conclusion}\label{sec:conclusion}

Model counting is intrinsically hard, hence, approximate techniques have been
developed to scale beyond what exact counters can do.
We propose a new approximate counter, \PartialKC{}, based on our partial CCDD KC form.
It is anytime and able to converge to the exact count. 
We present many techniques exploiting partial CCDD to achieve better scalability.
Our experimental results show that \PartialKC{} is more accurate than existing anytime approximate counters satss and STS, and scales better.
Surprisingly, \PartialKC{} is able to outperform recent state-of-art exact counters by reaching convergence on many instances.

\section*{Acknowledgments}
We are grateful to the anonymous reviewers for their constructive feedback.
We thank Mate Soos and Stefano Ermon for providing their tools. 
This work was supported in part by the National Research Foundation Singapore under its NRF Fellowship Programme [NRF-NRFFAI1-2019-0004] and the AI Singapore Programme [AISG-RP-2018-005], NUS ODPRT9 Grant [R-252-000-685-13], Jilin Province Natural Science Foundation [20190103005JH] and National Natural Science Foundation of China [61976050]. 
The computational resources were provided by
the National Supercomputing Centre, Singapore (\url{https://www.nscc.sg}).

\bibliography{papers,software,books}

\end{document}